%% file: preprint.tex
\documentclass{article}

\usepackage{microtype}
\usepackage{graphicx}
\usepackage{subcaption}
\usepackage{booktabs} 

\usepackage[backref=page]{hyperref}
\usepackage{tcolorbox}
\usepackage{enumitem}

\usepackage[preprint]{icml2026}

\usepackage{amsmath}
\usepackage{amssymb}
\usepackage{mathtools}
\usepackage{amsthm}

\usepackage{multicol}
\usepackage{multirow}

\newcommand{\Ours}[0]{\textcolor{blue}{\textbf{Ours }}}

\usepackage[capitalize,noabbrev]{cleveref}

\theoremstyle{plain}
\newtheorem{theorem}{Theorem}[section]
\newtheorem{proposition}[theorem]{Proposition}

\theoremstyle{definition}
\newtheorem{definition}[theorem]{Definition}

\theoremstyle{remark}

\usepackage[textsize=tiny]{todonotes}

\icmltitlerunning{\textsc{DistillLens}: Symmetric Knowledge Distillation Through Logit Lens}

\begin{document}

\twocolumn[
  \icmltitle{\textsc{DistillLens}: Symmetric Knowledge Distillation Through Logit Lens}



  \icmlsetsymbol{equal}{*}

  \begin{icmlauthorlist}
    \icmlauthor{Manish Dhakal}{gsu}
    \icmlauthor{Uthman Jinadu}{gsu}
    \icmlauthor{Anjila Budathoki}{gsu}
    \icmlauthor{Rajshekhar Sunderraman}{gsu}
    \icmlauthor{Yi Ding}{auburn}
  \end{icmlauthorlist}

  \icmlaffiliation{gsu}{Georgia State University, GA, USA}
  \icmlaffiliation{auburn}{Auburn University, AL, USA}

  \icmlcorrespondingauthor{Manish Dhakal}{mdhakal3@gsu.edu}
  \icmlcorrespondingauthor{Yi Ding}{yiding@auburn.edu}

  \icmlkeywords{Machine Learning, ICML}

  \vskip 0.3in
]



\printAffiliationsAndNotice{}  

\begin{abstract}
Standard Knowledge Distillation (KD) compresses Large Language Models (LLMs) by optimizing final outputs, yet it typically treats the teacher's intermediate layer's thought process as a black box. 
While feature-based distillation attempts to bridge this gap, existing methods (e.g., MSE and asymmetric KL divergence) ignore the rich uncertainty profiles required for the final output. 
In this paper, we introduce \textsc{DistillLens}, a framework that symmetrically aligns the evolving thought processes of student and teacher models. By projecting intermediate hidden states into the vocabulary space via the Logit Lens, we enforce structural alignment using a symmetric divergence objective.
Our analysis proves that this constraint imposes a dual-sided penalty, preventing both overconfidence and underconfidence while preserving the high-entropy information conduits essential for final deduction.
Extensive experiments on GPT-2 and Llama architectures demonstrate that \textsc{DistillLens} consistently outperforms standard KD and feature-transfer baselines on diverse instruction-following benchmarks.
The code is available at  \url{https://github.com/manishdhakal/DistillLens}.
\end{abstract}

\section{Introduction}

\begin{figure}[t]
    \centering
    \includegraphics[width=0.9\linewidth]{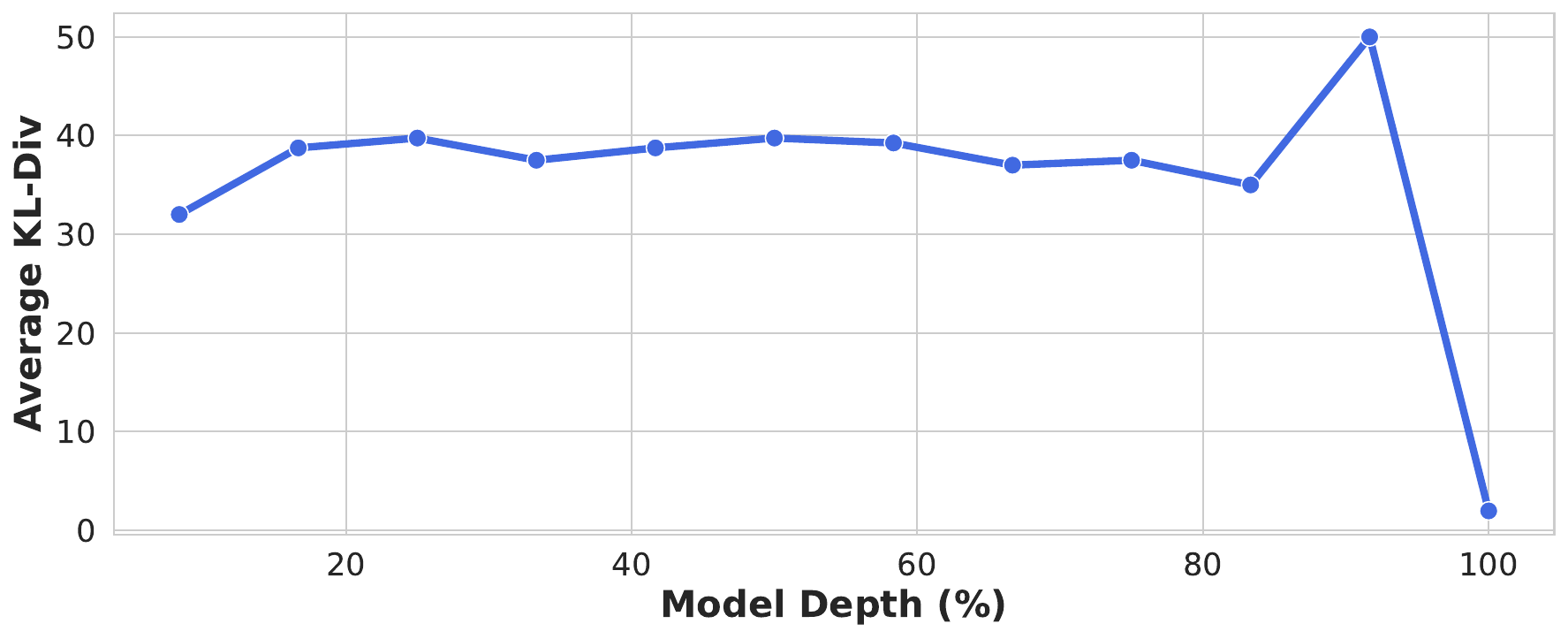}
    \caption{The distilled student (GPT2-120M) notably diverges from the teacher model for the hidden layers, except for the final layer.}
    \label{fig:kl_div}
\end{figure}

Recent studies have shown that the intermediate layers of language models encode richer information, contributing to the model's final predictions and enhancing performance on downstream tasks~\cite{skean2025layer,van2019does,zhang2024investigating}.
We refer to these depth-wise representations as \textit{evolving thought processes} that transform across layers to yield the final outputs.
Traditional Knowledge Distillation (KD)~\cite{hinton2015distilling} ignores the evolving thought processes by fundamentally treating the teacher model as a ``black box'' instructor, transferring knowledge solely through the final output distribution.
Consequently, the thought process of the student model diverges notably from teacher model, despite converging at the final layer, as illustrated in~\Cref{fig:kl_div}.
By restricting supervision to the final logits, standard KD forces the student to reverse-engineer the teacher's complex reasoning process from a sparse signal, effectively hiding the thought process that occurs across the teacher's depth.
This approach is analogous to a student memorizing the final answer to a complex mathematical proof without understanding the derivation steps, which often leads to a failure in generalization and an inferior-performing model.

The thought processes of Large Language Models (LLMs) have been explained by the recent literature in mechanistic interpretability with tools like the logit lens~\cite{nostalgebraist2020interpreting}:
\begin{align}
    p^{(l)} (y | x) & =\text{LogitLens}(h^{(l)}, W_U) \\ 
        &= \text{softmax}(W_U h^{(l)}),\label{eq:logit_lens}
\end{align}
where hidden state $h^{(l)} \in \mathbb{R}^{d}$ for the $l^{th}$ intermediate layer, and $W_U \in \mathbb{R}^{V \times d}$ is the unembedding matrix.

While some distillation methods attempt to transfer intermediate features, they face two critical limitations.
\textbf{(1)} Typical methods rely on minimizing Mean Squared Error (MSE), $\mathcal{L}_{MSE} = \|W_s h_p - h_{q_\theta}\|^2$, where $W_s$ projects the teacher hidden state $h_p$ to the space of student hidden state $h_{q_\theta}$~\cite{romero2014fitnets,jiao2021improving}.
We argue that this approach is \textbf{divergence insensitive}, as MSE assumes an isotropic embedding space where all error directions are equal.
However, high-dimensional language embedding spaces are highly anisotropic~\cite{ethayarajh2019contextual}; identical MSE values can yield vastly different probability divergences depending on whether the error projects onto high-probability or low-probability tokens at $p^{(l)} (y | x)$.
\textbf{(2)} To address this semantic mismatch, direct minimization of asymmetric Kullback-Leibler divergence $\mathcal{L}_{KL}(p\|q_\theta)$ has been proposed~\cite{sun2019patient,gong2025beyond}.
However, the asymmetric nature of this metric creates a new alignment issue.
By definition, $\mathcal{L}_{KL}(p\|q_\theta)$ heavily penalizes underestimation of the teacher's high-probability tokens while potentially \textbf{ignoring the overestimation of lower-probability tails}, which is a ``mean-seeking'' behavior~\cite{wu2024rethinking}. 

To address these issues, we propose \textbf{\textsc{DistillLens}}, a distillation framework that leverages symmetric divergence objectives, such as Jensen-Shannon Divergence (JSD), to align the evolving thought processes of the student and teacher.
First, we project the hidden layers to the vocabulary space by using logit lens, as in \Cref{eq:logit_lens}.
Then, we use symmetric objective to match token distribution of that space. 
Unlike asymmetric objectives that prioritize either mode-seeking or mean-seeking behavior, our approach recognizes intermediate layers as uncertainty information conduits.
Symmetric objectives penalize the student model for both underestimating and overestimating the teacher's probability distribution which reduces the distribution mismatch for all probabilities (see \cref{sec:theoretical_analysis}).

In summary, our main findings form this work are as follows:
\begin{itemize}
    \item We introduce \textsc{DistillLens}, a novel framework that symmetrically distills the evolving thought process of LLMs by supervising intermediate layer distributions projected through the Logit Lens.
    \item We provide a theoretical derivation of the symmetric loss landscape, proving it enforces a dual-sided penalty that prevents both overconfident and underconfident output probability sides.
    \item We show through extensive experiments that \textsc{DistillLens} consistently outperforms standard KD baselines and existing feature-transfer methods across diverse language modeling datasets.
\end{itemize}

\section{Related Works}

\subsection{Knowledge Distillation (KD)}
Knowledge distillation approaches can be broadly categorized into off-policy and on-policy distillations. 
Standard methods generally operate via off-policy distillation, where the student learns from the teacher's logits on fixed, ground-truth datasets~\cite{hinton2015distilling,sanh2019distilbert,wu2024rethinking,kim2016sequence,wen2023f}.
While efficient, this setup suffers from exposure bias~\cite{arora2022exposure}, as the student never comes across its own incorrect generation during training.
Conversely, on-policy approaches~\cite{agarwal2024onpolicy,gu2024minillm,ko2024distillm,kodistillm} mitigate this training-inference mismatch by fine-tuning the student-generated output (SGO) sequences, but these methods suffer from self-generative inefficiency during training.
None of these distillation approach concern about the intermediate features distillation.
Thus, we propose \textsc{DistillLens} as a novel modular intermediate matching that can be applied to any of the distillation approaches.


\subsection{Interpretability \& Logit Lens}
Understanding the internal information flow of transformers has been a focal point of recent mechanistic interpretability research. The logit lens technique, popularized by \citet{nostalgebraist2020interpreting}, posits that intermediate hidden states can be interpreted by projecting them onto the model's pre-trained unembedding matrix. 
Subsequent studies have verified that transformers refine their predictions iteratively across layers~\cite{geva2021transformer,elhage2021mathematical}, acting as an evolving belief state. \citet{belrose2023eliciting} and \citet{halawi2024overthinking} have utilized these insights to probe layer-wise confidence and dynamic halting mechanisms.
While these tools have primarily been used for post-hoc interpretability, \textsc{DistillLens} repurposes the logit lens to be used for active supervision during training, ensuring the student's internal trajectory aligns semantically with the teacher's.

\subsection{Feature-based Distillation}
Feature-based KD attempts to align the intermediate representations of the student and teacher directly. 
\citet{romero2014fitnets} introduced FitNets, which utilize regression losses to match the hidden states of intermediate layers. Subsequent works have proposed aligning attention maps~\cite{zagoruyko2017paying,wang2020minilm}.
Recently, \citet{sun2019patient} and \citet{gong2025beyond} proposed to distill the intermediate features from teacher to student using the asymmetric Kullback-Leibler divergence $\mathcal{L}_{KL}(p||q_\theta)$ metric, which undervalues the matching of low-probability regions.
\textsc{DistillLens} overcomes this by using a symmetric loss function that values both low and high probability regions.

\section{\textsc{DistillLens}}\label{sec:distill_lens}
\begin{algorithm}[t]
   \caption{\textsc{DistillLens} Training (with JSD)}
   \label{alg:distilllens}
\begin{algorithmic}
   \STATE {\bfseries Input:} Dataset $\mathcal{D}$, teacher $p$, student $q_{\theta}$, layer mapping $(l, l') \in \mathcal{M}$ from $q_\theta \to p$ , scaling factor $\lambda=1.0$.
   \STATE {\bfseries Output:} Trained student parameters $\theta_{N}$

   \FOR{training step $t=1$ to $N$}
       \STATE Sample batch of prompts $x$ from $\mathcal{D}$
       \STATE Get intermediate states $\{h_p\}$ and $\{h_{q_{\theta}}\}$.
           \STATE $q_\theta^{(l)} =\text{LogitLens}(h_{q_\theta}^{(l)})$
           \STATE $p^{(l')} =\text{LogitLens}(h_p^{(l')})$
        \STATE $\mathcal{L}_{inter} = \frac{1}{|\mathcal{M}|} \sum\limits_{(l, l') \in \mathcal{M}} \mathcal{L}_{JSD} (p^{(l') }, q_\theta^{(l)} )$

       \STATE Compute Task Loss $\mathcal{L}_{task}$ (e.g., standard KD)
       \STATE $\mathcal{L}_{total} \leftarrow \mathcal{L}_{task} + \lambda \cdot \mathcal{L}_{inter}$
       \STATE Update $\theta$ by descending $\nabla_{\theta} \mathcal{L}_{total}$
   \ENDFOR
\end{algorithmic}
\end{algorithm}

\begin{figure*}[t]
    \centering
    \includegraphics[width=0.8\linewidth]{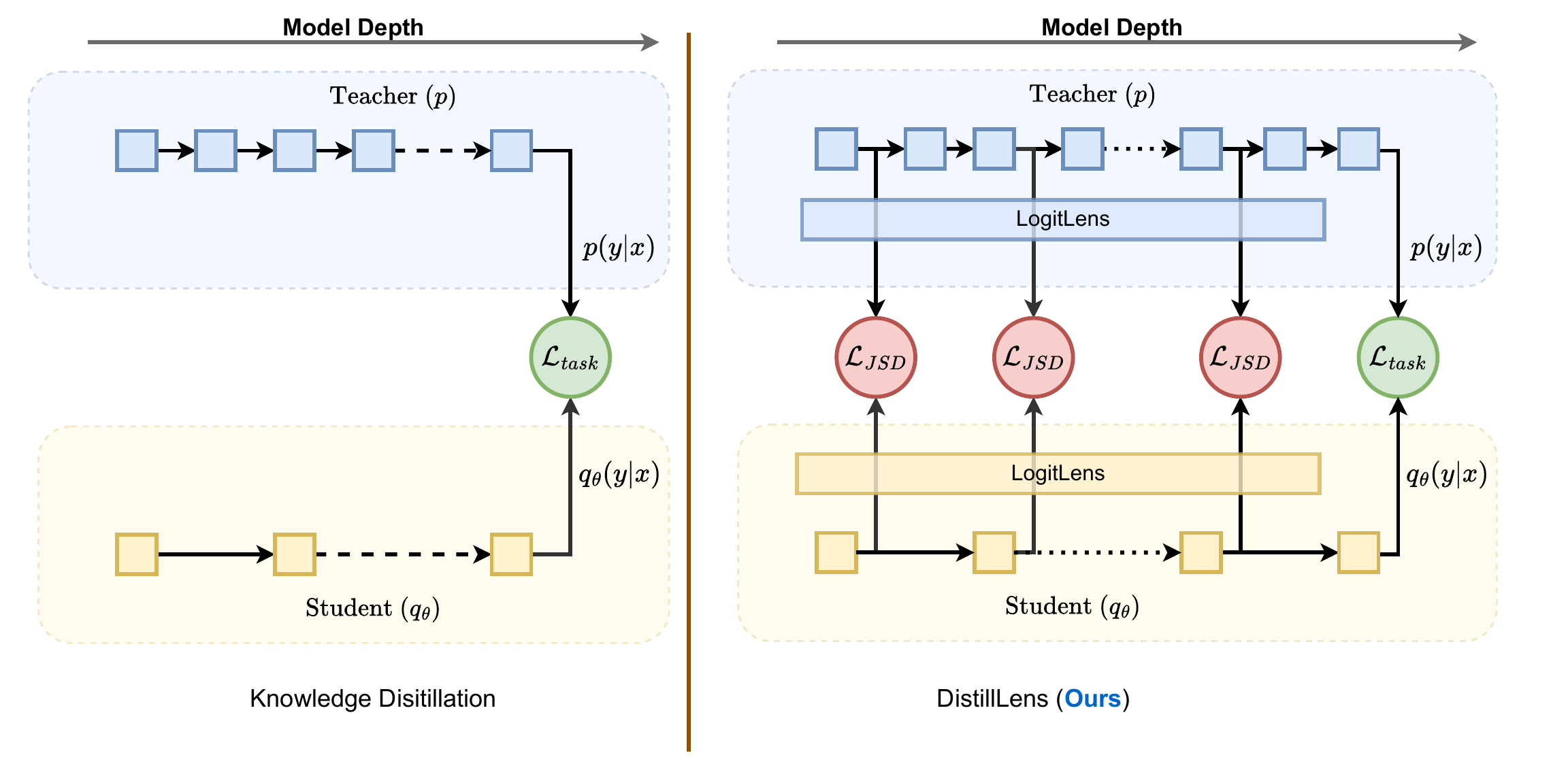}
    \caption{\textbf{The \textsc{DistillLens} Framework.} Comparison between standard Knowledge Distillation (\textbf{left}) and our proposed \textsc{DistillLens} approach (\textbf{right}). Unlike standard KD, which restricts supervision solely to the final output logits, \textsc{DistillLens} aligns the intermediate thought processes of the student and teacher. By projecting intermediate hidden states into the vocabulary space using the Logit Lens, we compute the symmetric divergence loss ($\mathcal{L}_{JSD}$), which is optimized jointly with the standard task loss ($\mathcal{L}_{task}$)}
    \label{fig:architecture}
\end{figure*}

The core principle of \textsc{DistillLens} is to move beyond final-layer supervision by aligning the internal thought process of the student with that of the teacher. 
As outlined in \Cref{alg:distilllens} and \cref{fig:architecture}, our method functions as a modular objective $\mathcal{L}_{inter}$ that can be integrated with any task-specific  $\mathcal{L}_{task}$. 
The process involves projecting hidden states into a shared vocabulary space, calculating layer-wise divergence via symmetric distillation, and backpropagating a weighted combination of task and intermediate alignment losses.

\subsection{Theoretical Framework}

\subsubsection{The Necessity of Symmetric Alignment}
The final layers of LLMs generally ignore lower probabilities as noise for clean argmax prediction or sampling over the top-$k$ tokens.
In contrast, for the hidden layers, we consider these probabilities a richer source of semantic uncertainty information that contributes to the evolving thought process required for final-layer prediction.
Consequently, accurately mapping these regions is essential for replicating the teacher's thought process.

However, with standard asymmetric objectives, smaller student models struggle to achieve precise alignment in these areas due to their opposing alignment behaviors.
Standard Forward KL (FKL) exhibits a ``pull-up'' effect~\cite{wu2024rethinking,kodistillm}: its \textit{mean-seeking} nature forces the student to cover all teacher possibilities, drifting toward overestimation and resulting in an over-smoothed high-entropy state. 
Conversely, Reverse KL (RKL) exerts a ``push-down'' effect~\cite{wu2024rethinking,kodistillm}: its \textit{mode-seeking} nature aggressively penalizes the assignment of probability to lower-probability signals, drifting toward underestimation and suppression of nuance. 
Symmetric distillation resolves this by balancing these opposing forces to drive the student toward perfect alignment, as we analyze in the next section.

\subsubsection{Analysis of the Loss Landscape}\label{sec:theoretical_analysis}
To supervise the projected distributions of hidden states, we introduce the loss $\mathcal{L}_{inter}$, which is the average of symmetric divergences of the states. 
\textsc{DistillLens} primarily utilizes JSD as its symmetric objective, though other symmetric metrics like Jeffreys Divergence (JD) \cite{jeffreys1948theory} are possible variants.


\begin{definition}[Confidence Score]
Let $p(y|x)$ and $q_{\theta}(y|x)$ be the output distributions of the teacher and student models, respectively, for input prompts $x$. 
We define the confidence score $c_{\theta}(y|x)$ as the ratio of the student's probability to the teacher's probability:
\begin{equation}\label{eq:confidence}
    c_{\theta}(y|x) = \frac{q_{\theta}(y|x)}{p(y|x)}
\end{equation}
\end{definition}

We use the score to evaluate the loss at two extremities: overconfidence ($c_{\theta} \to \infty$) and underconfidence ($c_{\theta} \to 0$).
Both of these are undesired cases for the teacher-to-student alignment.

\begin{definition}[JSD]
We use the standard definition of $\mathcal{L}_{JSD}$ as a symmetric objective, measuring the average Kullback-Leibler divergence from teacher $p$ and student $q_\theta$ to their mixed distribution $m(y|x) = \frac{1}{2}(p(y|x) + q_\theta (y|x))$:
\begin{equation}\label{eq:jsd}
    \mathcal{L}_{JSD}(p, q_\theta) = \frac{1}{2} \left[ \mathcal{L}_{KL}(p \| m) + \mathcal{L}_{KL}(q_\theta \| m)\right].
\end{equation}
\end{definition}

\begin{proposition}[Dual-sided Alignment]
The Jensen-Shannon Divergence objective $\mathcal{L}_{JSD}$ aligns distributions via a dual-sided loss landscape. 
It linearly penalizes overconfidence ($c_{\theta} \to \infty$) and applies a bounded penalty for underconfidence ($c_{\theta} \to 0$), effectively trying to align with the teacher's supervision ($c_{\theta} \to 1$).
\end{proposition}

\begin{proof}
For brevity, we omit the conditioning on the input $x$ and output $y$. Substituting the definition of $c_\theta$ and the mixture distribution $m = \frac{1}{2}(p + q_\theta)$ into the expanded form of $\mathcal{L}_{JSD}$, we derive:
\begin{align}
    \mathcal{L}_{JSD}(p,q_\theta) &= \frac{1}{2} \mathbb{E}_{p} \Big[ \underbrace{c_\theta \log c_\theta - (1+c_\theta) \log \frac{1+c_\theta}{2}}_{g(c_\theta)} \Big]
\end{align}
The complete derivation is provided in Appendix \cref{app:sec:JSD_derivation}.

To analyze the optimization landscape, we decompose the objective into the per-class objective $g(c_\theta)$. 
We treat the teacher probability $p$ of $\mathbb{E}_p$ as a static scaling factor and analyze the behavior of $g(c_\theta)$ in three distinct cases:

\textbf{Case 1: Overconfidence ($c_{\theta} \to \infty$).} As the student assigns excessive probability mass relative to the teacher, the term behaves linearly. JSD applies a controlled linear penalty:
\[
    \lim_{c_{\theta} \to \infty} g(c_\theta) \approx c_\theta \log 2 \quad \text{(Linear Hallucination Penalty)}
\]

\textbf{Case 2: Underconfidence ($c_{\theta} \to 0$).} As the student fails to capture the teacher's probability mass, the term approaches a finite constant (since $\lim_{c \to 0} c \log c = 0$). JSD imposes a saturated ceiling on the penalty:
\[
    \lim_{c_{\theta} \to 0} g(c_\theta) = \log 2 \quad \text{(Bounded Missed Recalls Penalty)}
\]

\textbf{Case 3: Perfect Alignment ($c_{\theta} = 1$).} When the student perfectly matches the teacher ($q_\theta = p$), the loss vanishes, confirming $c_\theta=1$ as the global minimum:
\[
    g(1) = 1 \cdot \log 1 - (2) \log 1 = 0
\]
Thus, $\mathcal{L}_{JSD}$ enforces a convex, dual-sided optimization path that drives $c_\theta \to 1$.
\end{proof}

However, earlier works~\cite{ko2024distillm,wu2024rethinking} have shown that models are penalized by $\mathcal{L}_{KL}(p \|q_\theta)$ for being underconfident only in high target probabilities and by reverse $\mathcal{L}_{(R)KL}(q_\theta \| p)$ for being overconfident only in low target probabilities.

\subsection{Framework Implementation}
\label{sec:distill_lens for_intermediate_layers}

Building on the theoretical properties of symmetric divergence, we formulate the practical training objective for \textsc{DistillLens}. We consider a teacher model $p$ with $L_T$ layers and a student model $q_\theta$ with $L_S$ layers, where typically $L_S < L_T$.

\paragraph{Layer Mapping.}
To align the thought trajectory effectively, we employ a uniform mapping strategy $\mathcal{M}$ that associates student layers with teacher layers at regular intervals.
We select a subset of $K$ intermediate layers to distill for the student model. 
For each selected student layer index $l \in \{1, \dots, L_S\}$, the corresponding teacher layer index $l'$ is determined by maintaining a proportional depth ratio:
\begin{equation}
    l' = \text{Round}\left( l \times \frac{L_T}{L_S} \right).
\end{equation}
This uniform stride ensures that \textsc{DistillLens} captures the evolution of hidden states across the model's full depth, preventing the student from skipping critical deduction steps.
For each pair $(l, l') \in \mathcal{M}$, we project the hidden states $h_{q_\theta}^{(l)}$ and $h_{p}^{(l')}$ into the vocabulary probability space by utilizing their corresponding unembedding matrices $W_{U_{q_\theta}}$ and $W_{U_p}$, yielding $q_\theta^{(l)}$ and $p^{(l')}$.

\paragraph{Optimization Objective.}
We treat the intermediate alignment as a regularization term that constrains the student's internal state. The intermediate loss is computed as the average JSD across all mapped layers:
\begin{equation}
    \mathcal{L}_{inter} = \frac{1}{|\mathcal{M}|} \sum_{(l, l') \in \mathcal{M}} \mathcal{L}_{JSD}\left(p^{(l')}(y|x), q_\theta^{(l)}(y|x)\right).
\end{equation}
The total training objective combines the standard task loss, typically KL divergence on the final logits, $\mathcal{L}_{KD}$) with our structure-aware intermediate loss:
\begin{equation}
    \mathcal{L}_{total} = \mathcal{L}_{task} + \lambda \cdot \mathcal{L}_{inter},
\end{equation}
where $\lambda$ is a scalar hyperparameter controlling the strength of the intermediate supervision.
This formulation forces the student to not only match the final prediction but to arrive at it through a sequence of probability distributions that mirror the teacher's thought process, as detailed in \Cref{alg:distilllens}.

\section{Experiments}
\label{sec:experimental_setup}

\paragraph{Models and Architecture.}
To analyze the efficacy of \textsc{DistillLens} across different scales, we conduct knowledge distillation strictly within model families. We employ GPT-2-1.5B (XL) as the teacher for GPT-2-120M (base) and GPT-2-340M (medium) students, and Llama-7B as the teacher for the TinyLlama-1.1B student.

\paragraph{Implementation Details.}
We perform knowledge distillation using the \texttt{databricks-dolly-15k} dataset~\cite{gu2024minillm}. 
For the layer mapping $\mathcal{M}$, we specifically map student layers 
$\{2, 4, 6, 8, 10\}$ to teacher layers $\{8, 16, 24, 32, 40\}$ for GPT-2 (120M), 
and $\{4, 8, 12, 16, 20\}$ to $\{8, 16, 24, 32, 40\}$ for GPT-2 (340M). 
For the TinyLlama experiments, we map student layers $\{4, 7, 11, 15, 18\}$ to teacher layers $\{5, 10, 16, 21, 26\}$.
To optimize training efficiency, we utilize BF16 mixed precision on 4 NVIDIA A100 GPUs. 
The maximum sequence length is set to 512 tokens.
Model-specific hyperparameters are given in \cref{app:sec:training_details} of the Appendix.

\paragraph{Distillation Configuration.}
For \textsc{DistillLens}, we align the student's intermediate thought process by mapping $6$ equally spaced student layers (including the final layer) to their corresponding teacher layers based on depth proportions.
We employ the symmetric $\mathcal{L}_{JSD}$ as the minimization objective for these intermediate projections.
Earlier works~\cite{gu2024minillm,wu2024rethinking} have shown that reverse $\mathcal{L}_{(R)KL}(q_\theta \| p)$ favours instruction following tasks; thus, we use it as $\mathcal{L}_{task}$ for our experimentation.
The loss coefficients for the $\mathcal{L}_{inter}$ is set to $\lambda = 1.0$ after ablation (~\cref{app:sec:ablation_study}).

\paragraph{Evaluation Benchmarks.}
For generation during evaluation, we use standard sampling parameters with temperature $T=1.0$ and top-$p=1.0$ and with seed $\in \{10, 20, 30, 40, 50\}$. 
Our evaluation spans five diverse benchmarks: a held-out \textbf{DollyEval} set (500 samples), the user-centric \textbf{SelfInst}~\cite{wang2023self}, and the reasoning-focused \textbf{VicunaEval}~\cite{zheng2023judging}. Additionally, we evaluate long-form generation capabilities using the  response subsets of \textbf{S-NI}~\cite{wang2022super} and \textbf{UnNI}~\cite{honovich2023unnatural}.

\paragraph{Metrics.}
We assess the quality of model-generated responses using \textbf{Rouge-L} (R-L)~\cite{lin2004rouge}. 
Prior studies~\cite{wang2022super} have demonstrated that Rouge-L correlates well with human evaluation for large-scale instruction-following tasks, making it a suitable proxy for measuring generation precision and recall against ground truth references.
We estimate the semantic similarity with the ground truth using GPT-4o-mini~\cite{hurst2024gpt} as the judge and the SBERT score using Sentence-BERT models~\cite{reimers2019sentence}.
The GPT-4o-mini evaluation prompts are modified from \citep{zheng2023judging} (further explained in ~\cref{app:sec:gpt4o_mini_judge} of the Appendix) .
Due to page limitation constraint, we have included the SBERT evaluation in the Appendix \cref{app:sec:sbert_similarity_score}.

\paragraph{Baselines.}
We evaluate \textsc{DistillLens} against seven baselines using a teacher fine-tuned on databricks-dolly-15k as the reference. 
Our comparison spans standard approaches that employ final layer distillation, including Supervised Fine-Tuning (SFT) and Standard KD (Forward KL)~\cite{hinton2015distilling}, alongside Sequence-Level KD (SeqKD)~\cite{kim2016sequence}. 
To analyze the impact of additional loss formulation, we also benchmark against advanced divergence objectives: the mode-seeking Reverse KL (RKL), symmetric metrics like Jeffreys Divergence (JD)~\cite{jeffreys1948theory} and Jensen-Shannon Divergence (JSD), and the hybrid Adaptive KL (AKL)~\cite{wu2024rethinking}, which dynamically balances mean- and mode-seeking behaviors.

\section{Results \& Analysis}\label{sec:result_analysis}
\subsection{Main Results}\label{sec:main_results}

\input{tabs/main_results}
Since \textsc{DistillLens} is modular, we test it in combination with standard KD baselines (see \cref{fig:distillation_barplot}).
It consistently improves all of the baselines, but we observe that RKL is favoured the most; thus, we employ $\mathcal{L}_{(R)KL}$ as the $\mathcal{L}_{task}$.
\cref{tab:main_results} presents the superrior performance of \textsc{DistillLens} against state-of-the-art (SOTA) baselines across three student model architectures. 
Our method consistently outperforms all baseline approaches, including standard KD and advanced divergence-based methods like AKL. 
For the GPT-2-340M student, \textsc{DistillLens} achieves an average Rouge-L score of 23.72, marginally surpassing the teacher model's performance of 23.52. 
In the GPT-2-120M experiments, our approach yields a substantial improvement, raising the average score to 21.12 compared to 17.74 for standard KD and 16.56 for SFT. 
We observe consistent gains in the Llama family as well; when distilling Llama-7B into the TinyLlama-1.1B student, \textsc{DistillLens} reaches an average score of 25.48. 
This performance surpasses the standard KD baseline of 23.82 by over 1.6 points and outperforms the strongest competitive baseline (FKL+RKL) at 24.56, validating that aligning intermediate thought trajectories provides a robust supervision signal across diverse architectures.

\subsection{On-policy Distillation}\label{sec:on_policy_results}
\begin{figure}[t]
  \centering
    \includegraphics[width=0.9\linewidth]{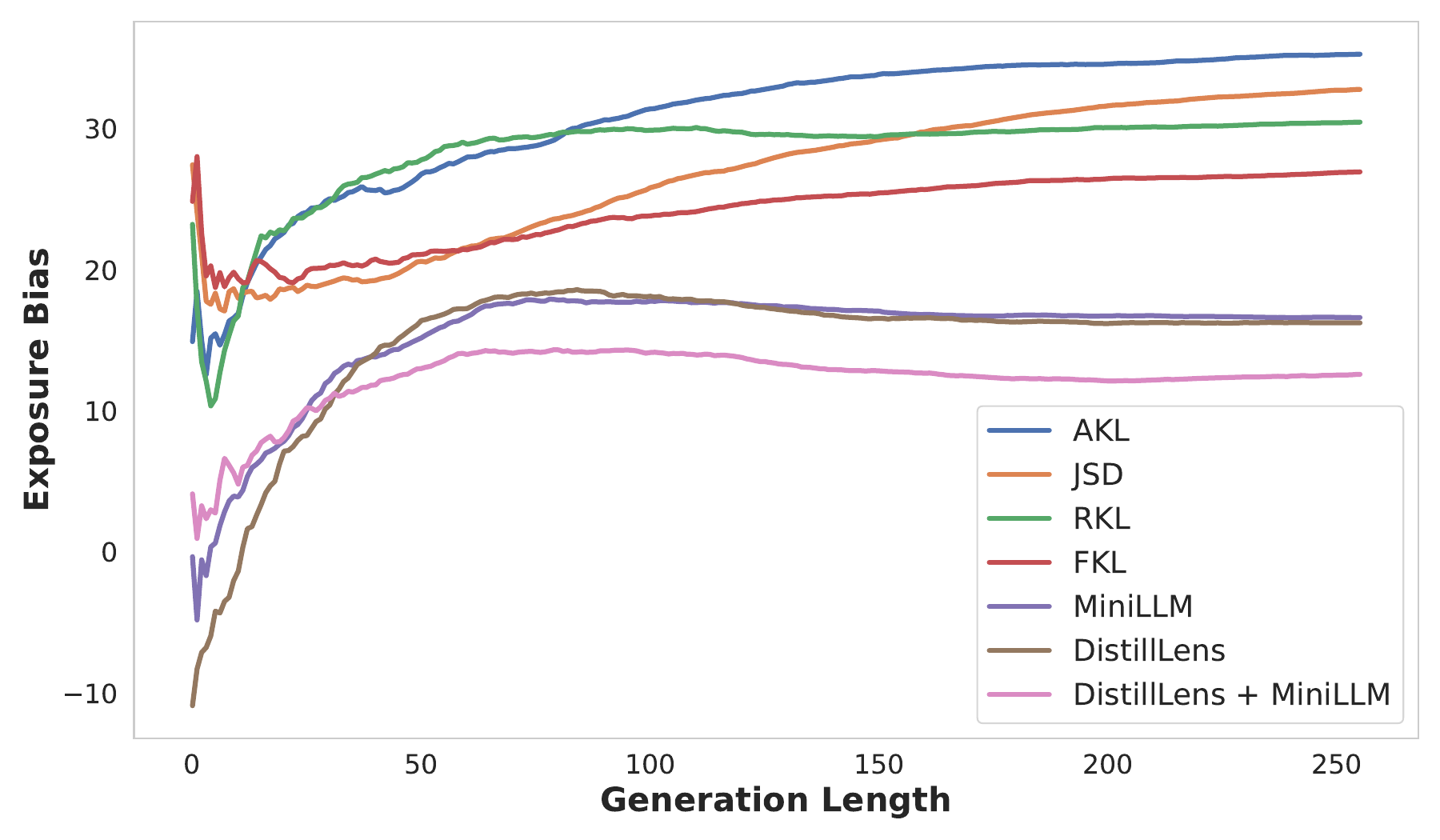}
    \caption{\textbf{Exposure bias \textit{vs.} generated sequence length.} Lower bias indicates that the model is more robust to its own generation errors. Please refer to Appendix \cref{app:sec:exposure_bias} for calculating the bias.}
    \label{fig:exposure_bias}
\end{figure}

\begin{table}[h]
    \centering
    \caption{\textbf{Comparison against on-policy methods (R-L score).} While on-policy methods give better results than off-policy, they usually suffer from slower training speed (see ~\cref{fig:speed_comparison_barplot}). \textsc{DistillLens} achieves competitive results efficiently, and further improves performance when combined with MiniLLM (Hybrid).}
    \resizebox{\linewidth}{!}{
        \begin{tabular}{l|c|ccccc|c}
        \toprule
        \textbf{Method} & \textbf{Policy} & \textbf{Dolly} & \textbf{SelfInst} & \textbf{Vicuna} & \textbf{S-NI} & \textbf{UnNI} & \textbf{Avg} \\
        \midrule
        KD (FKL) & \textit{Off} & 23.5 & 10.3 & 14.7 & 16.6 & 20.9 & 17.20 \\
        \midrule
        GKD & \multirow{3}{*}{\textit{On}} & 24.4 & 10.4 & 15.7 & 17.2 & 19.9 & 17.51 \\
        DistiLLM & & 24.1 & 10.8 & 17.1 & 17.6 & 19.8 & 17.88 \\
        MiniLLM & & 24.6 & \underline{13.2} & \underline{16.9} & \underline{25.3} & 26.6 & \underline{21.32} \\
        \midrule
        \textsc{DistillLens} & \textit{Off} & \textbf{25.2} & 12.4 & 15.8 & 24.3 & \textbf{27.9} & 21.12 \\
        \quad + MiniLLM & \textit{Hybrid} & \underline{25.1} & \textbf{13.3} & \textbf{17.8} & \textbf{26.1} & \underline{27.2} & \textbf{21.86} \\
        \bottomrule
        \end{tabular}
    }
    \label{tab:on_policy}
\end{table}

The main results reported above rely on off-policy distillation using fixed teacher-generated outputs. 
By optimizing on static sequences, these methods inherently suffer from exposure bias~\cite{arora2022exposure}: a distribution mismatch arises because the model is trained on ground-truth sequences but evaluated on its own autoregressive generations~\cite{gu2024minillm}. 
This increased exposure bias directly degrades the generation capabilities of language models~\cite{arora2022exposure}. 
In contrast, on-policy distillation~\cite{gu2024minillm,agarwal2024onpolicy} mitigates this by training the student model $q_\theta$ on its own self-generated responses $y$ for given prompts $x$. 
However, this effectiveness comes at a high computational cost; the generation process introduces significant GPU overhead and slows down training~\cite{ko2024distillm,kodistillm}.

\textsc{DistillLens} addresses this trade-off effectively. As illustrated in \cref{fig:exposure_bias}, our approach significantly reduces exposure bias compared to standard KD without incurring the sampling overhead of self-generation. Consequently, \textsc{DistillLens} achieves training times significantly faster than on-policy baselines while maintaining competitive performance with state-of-the-art (SOTA) methods (see \cref{tab:on_policy}). 
Furthermore, to demonstrate modularity, we integrate \textsc{DistillLens} with MiniLLM~\cite{gu2024minillm}, a leading on-policy method. 
The resulting combination outperforms all existing baselines, confirming that our approach can effectively boost the performance of on-policy techniques.

\subsection{Intermediate Feature Transfers}

\begin{table}[t]
    \centering
    \caption{\textbf{Comparison against baseline intermediate feature transfers (R-L score).} \textsc{DistillLens} (\Ours) consistently outperforms standard feature distillation baselines. Between the two symmetric variants, JSD marginally outperforms JD.}
    \resizebox{\linewidth}{!}{
        \begin{tabular}{l|l|ccccc|c}
        \toprule
        \textbf{Model} & \textbf{Method} & \textbf{Dolly} & \textbf{SelfInst} & \textbf{Vicuna} & \textbf{S-NI} & \textbf{UnNI} & \textbf{Avg} \\
        \midrule
         \multirow{3}{*}{GPT-2-120M} & MSE & 24.3 & 12.3 & 15.7 & 23.2 & 26.4 & 20.38 \\
         & FDD & 23.9 & 10.8 & 15.1 & 17.9 & 22.0 & 17.94 \\
         & \Ours (JD) & \textbf{25.3} & 12.0 & 15.4 & 24.2 & 27.5 & 20.88 \\
         & \Ours (JSD) & 25.2 & \textbf{12.4} & \textbf{15.8} & \textbf{24.3} & \textbf{27.9} & \textbf{21.12} \\
         \midrule
         \multirow{3}{*}{GPT-2-340M} & MSE & 25.6 & 13.0 & 16.5 & 22.8 & 28.4 & 21.26 \\
         & FDD & 24.9 & 13.3 & 15.7 & 23.6 & 27.4 & 20.98 \\
         & \Ours (JD) & \textbf{26.5} & \textbf{14.7} & \textbf{17.4} & 27.7 & 31.3 & 23.52 \\
         & \Ours (JSD) & 26.4 & 14.6 & 16.5 & \textbf{28.1} & \textbf{33.0} & \textbf{23.72} \\
        \bottomrule
        \end{tabular}
    }
    \label{tab:interm_feat_transfer}
\end{table}

The most common approach to transfer intermediate features is the minimization of Mean Squared Error, $\mathcal{L}_{MSE} = \|W_s h_p - h_{q_\theta}\|^2$. Recently, \citet{gong2025beyond} proposed FDD, a method that aligns feature dynamics using asymmetric KL-divergence. We employ these two established methods as our primary baselines.

As shown in~\cref{tab:interm_feat_transfer}, \textsc{DistillLens} consistently outperforms both MSE and FDD across multiple model sizes and benchmarks. For instance, on GPT-2-120M, our method achieves an average R-L score of \textbf{21.12}, surpassing MSE (20.38) and FDD (17.94). This confirms that projecting features into the vocabulary space and enforcing symmetric alignment is superior to direct vector regression or asymmetric divergence.

We explicitly explore two variants of symmetric divergence: JSD and JD (FKL~+~RKL). 
Both metrics effectively align intermediate features and surpass existing baselines. 
This further strengthens our argument for the need for symmetric divergence for the features. 
Performance varies marginally between JD and JSD, exchanging the lead across different benchmarks.
For theoretical analysis of JD similar to JSD, please refer to Appendix ~\cref{sec:JD_derivation}.

\subsection{Ablation Study}
To determine the most effective layer selection pattern for the logit lens, we ablate the \textsc{DistillLens} framework. 
The study primarily distills to  GPT-2-120M with $\mathcal{L}_{(R)KL}$ as the $\mathcal{L}_{task}$, unless specified otherwise.
Additional ablation on the scaling factor $\lambda$ is provided in the Appendix \cref{app:sec:ablation_study}

\paragraph{Number of Logit Lens Layers.}

\begin{table}[t]
    \centering
    \caption{\textbf{Number of intermediate layers for \textsc{DistillLens} optimization.} ``0'' intermediate layers means a baseline with no intermediate layers matching. \textsc{DistillLens} peaks at 5 layers.}
    \resizebox{\linewidth}{!}{
        \begin{tabular}{r|c|cccccc}
        \toprule
         & \multicolumn{7}{c}{ Number of Logit Lens Layers}  \\
        \cmidrule{2-8}
         & 0 & 1 & 2 & 3 & 4 & 5 & 6 \\
        \midrule
        Avg R-L & $17.56$ & $20.62$ & $20.20$ & $20.00$ & $20.04$ & $\mathbf{21.12}$ & 20.60 \\
        $\Delta$ R-L & $+0.00$ & $+3.06$ & $+2.64$ & $+2.44$ & $+2.48$ & $\mathbf{+3.56}$ & 3.04 \\
        \bottomrule
        \end{tabular}
    }
    \label{tab:logit_lens_layers}
\end{table}

We initiate our study with a naive selection strategy, utilizing interleaved intermediate layers separated by a uniform gap.
As detailed in Table \ref{tab:logit_lens_layers}, applying \textsc{DistillLens} to a single intermediate layer (the mid-point) yields an immediate improvement of $+3.06$ in the Rouge-L score. 
Interestingly, this gain does not scale linearly; performance plateaus and slightly regresses between 2 and 4 layers, hovering near $20.00$. 
However, the performance rebounds significantly at 5 layers, achieving a peak score of $\mathbf{21.12}$ ($\mathbf{+3.56}$ over baseline). 
Extending the configuration to 6 layers exceeded the memory capacity of our 40GB A100 GPUs. 
Although we employed gradient accumulation to mitigate this specific overload, no further performance gains were observed. 
Consequently, we fix the number of intermediate layers at 5 for the remainder of our experiments.

\paragraph{Layers Selection Pattern.}

\begin{table}[t]
    \centering
    \caption{\textbf{R-L Score.} \textsc{DistillLens} performs better with an interleaved layers selection pattern than consecutive.}
    \resizebox{0.7\linewidth}{!}{
        \begin{tabular}{l|ccc}
        \toprule
        \textbf{Pattern} & \textbf{Dolly} & \textbf{S-NI} & \textbf{Un-NI}\\
        \midrule
        \textcolor{gray}{\textit{RKL (baseline)}} & 23.2 & 17.9 & 21.4\\
        First-5 & 24.9 & 23.8 & 26.0 \\
        Last-5 & 25.1 & 23.8 & 26.8 \\ 
        Interleaved & 25.2 & 24.3 & 27.9\\
        \bottomrule
        \end{tabular}
    }
    \label{tab:layer_selection_pattern}
\end{table}

We further assess the sensitivity of our method to the topology of layer selection by comparing interleaved versus consecutive configurations. 
With a fixed budget of 5 layers, we evaluate consecutive blocks at both model extremities: the first 5 layers and the last 5 layers (excluding the final layer covered by $\mathcal{L}_{task}$). 
As shown in Table \ref{tab:interm_feat_transfer}, all selection patterns yield substantial improvements over the RKL baseline. 
This universal gain underscores the robustness of symmetric divergence matching, demonstrating its efficacy regardless of the specific layers targeted. 
However, because the interleaved configuration consistently achieves the highest performance, we adopt it for all subsequent experiments.

\paragraph{Combining with Standard KDs.}
\begin{figure}[t]
    \centering
    \includegraphics[width=0.8\linewidth]{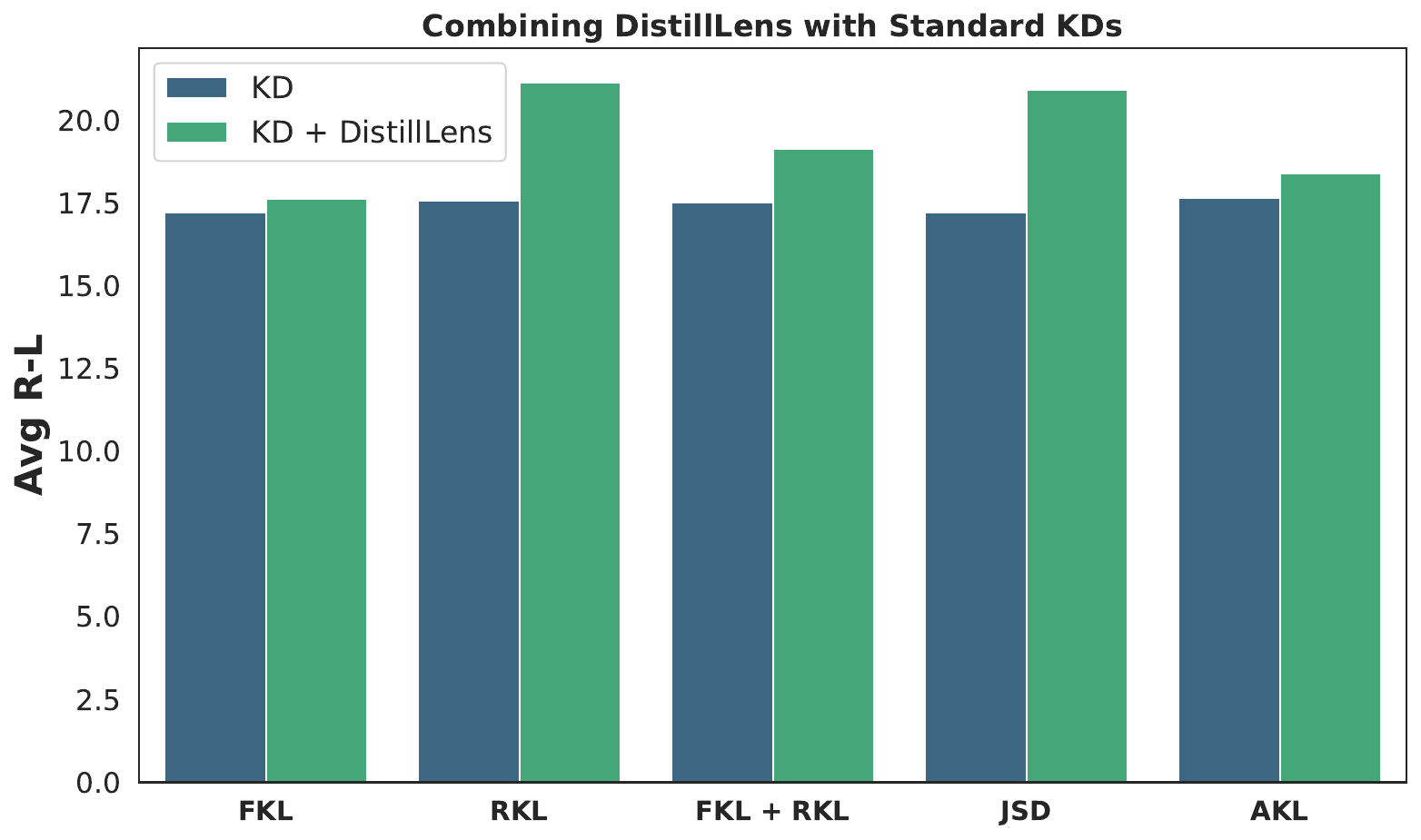}
    \caption{\textsc{DistillLens} incorporated with existing KD baselines. \textsc{DistillLens} improves the existing baselines.}
    \label{fig:distillation_barplot}
\end{figure}

To validate the modularity of \textsc{DistillLens}, we integrate it with existing standard distillation baselines, including FKL, RKL, JSD, and AKL. 
As illustrated in \cref{fig:distillation_barplot}, \textsc{DistillLens} consistently enhances the performance of existing KD baselines.
Notably, the improvement is most pronounced when combined with mode-seeking objectives.
While the performance gain for standard FKL is modest, integrating it with RKL and JSD yields substantial improvements, increasing the average R-L score from approximately 17.56 to 21.12 for RKL.
This suggests that our symmetric intermediate alignment is particularly effective at complementing objectives that utilize $\mathcal{L}_{(R)KL}$, providing the structural regularization needed to maximize their effectiveness.

\subsection{Further Discussions}

\label{sec:further_discussions}
\begin{figure}[t]
    \centering
    \includegraphics[width=0.9\linewidth]{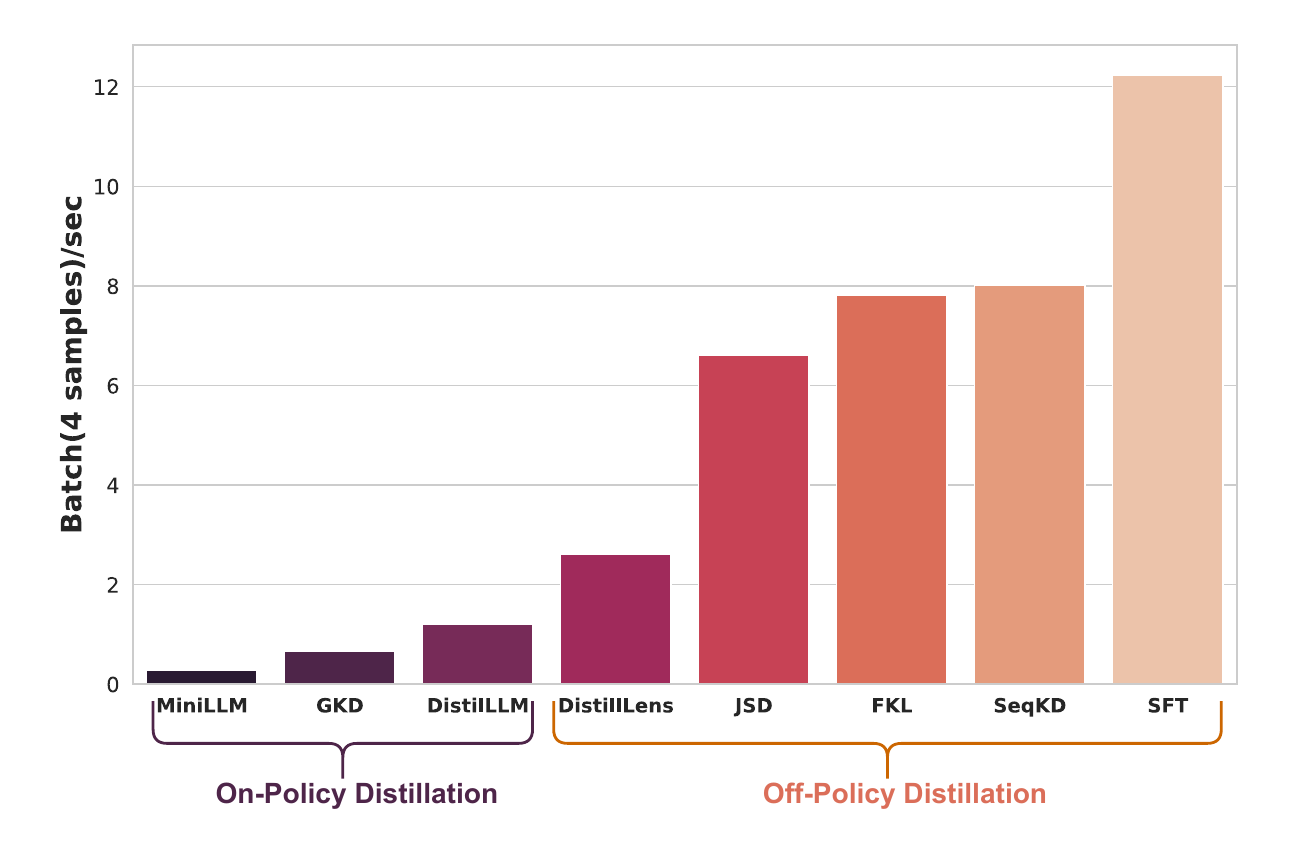}
    \caption{\textbf{Training Speed Comparison (per GPU):} We distill from GPT-2-1.5B to GPT-2-120M using A100 GPUs. Training with \textsc{DistillLens} is slower among the off-policy KDs, but faster when compared against on-policy methods.}
    \label{fig:speed_comparison_barplot}
\end{figure}

\paragraph{Training Overhead \textit{vs.} Inference Efficiency.}
Although projecting intermediate hidden states via $W_U$ introduces additional training computational costs, this overhead does not affect inference. 
The student model retains its original architecture, ensuring that deployment latency remains unchanged. 
In terms of training throughput (\cref{fig:speed_comparison_barplot}), \textsc{DistillLens} is more computationally intensive than standard off-policy baselines, but it delivers significantly better results (see \cref{tab:main_results}). 
However, it is faster than on-policy approaches while delivering comparable performance (see \cref{tab:on_policy}). 
Effectively, our method shifts the alignment cost entirely to the pre-deployment training phase, yielding a robust student model without compromising inference efficiency.

\paragraph{Reasoning Capabilities.}
On-policy approaches~\cite{ouyang2022training,shao2024deepseekmath,gu2024minillm} generally exhibit superior reasoning capabilities by conditioning models on self-generated tokens, a process that inherently reduces exposure bias but severely bottlenecks training speed. 
Our results indicate that \textsc{DistillLens} similarly mitigates exposure bias, suggesting a strong link between the alignment of internal ``thought processes'' and the model's external reasoning capabilities. 
Consequently, our framework provides a computationally efficient pathway for fine-tuning reasoning models, bridging the gap between fast off-policy training and the robustness typically associated with on-policy methods.

\section{Conclusion}
\label{sec:conclusion}
In this work, we introduce \textsc{DistillLens}, a novel distillation framework that aligns the intermediate thought trajectories of student and teacher models. By projecting hidden states into the vocabulary space via the Logit Lens and enforcing alignment via symmetric divergence objectives (e.g., JSD), we enable the student to faithfully mirror the teacher's internal deduction steps rather than merely replicating the final output. Empirical evaluations across GPT-2 and Llama architectures demonstrate that \textsc{DistillLens} consistently outperforms standard KD and existing feature-transfer methods.

\paragraph{Limitations and Future Directions.}
The primary limitation of our approach is the computational overhead incurred during training. Projecting multiple intermediate layers through the unembedding matrix $W_U$ scales with $\mathcal{O}(K \cdot V \cdot d)$, increasing memory usage and training time relative to standard KD.
We will work to make the training process more efficient in the future.
Future work will focus on leveraging \textsc{DistillLens} to inject complex reasoning capabilities into significantly smaller architectures. Additionally, we aim to extend this framework to cross-architecture (inter-family) distillation. While currently constrained by vocabulary mismatches between heterogeneous models, future advances in cross-tokenizer alignment or the adoption of universal vocabularies could allow \textsc{DistillLens} to distill knowledge from the largest available foundation models regardless of architecture.

\section*{Acknowledgements}
Research was sponsored by the Army Research Laboratory and was accomplished under Cooperative Agreement Number W911NF-23-2-0224. The views and conclusions contained in this document are those of the authors and should not be interpreted as representing the official policies, either expressed or implied, of the Army Research Laboratory or the U.S. Government. The U.S. Government is authorized to reproduce and distribute reprints for Government purposes notwithstanding any copyright notation herein.

\bibliography{references}
\bibliographystyle{icml2026}

\newpage
\appendix
\onecolumn

\section{Additional Theoretical Analysis}

\subsection{Derivation of $\mathcal{L}_{JSD}$ in terms of $c_\theta$} \label{app:sec:JSD_derivation}

For brevity, we omit the conditioning input $x$ and predicted output $y$ in the following (e.g., denoting $p(y|x)$ as $p$).
Substituting the definition of confidence score $c_\theta (y|x) = \frac{q_\theta(y|x)}{p(y|x)}$ and the mixture distribution $m(y|x) = \frac{1}{2}(p(y|x) + q_\theta (y|x))$ into the expanded form of $\mathcal{L}_{JSD}$, we derive:
\begin{align}
\mathcal{L}_{JSD}(p,q_\theta) &= \frac{1}{2} \left[ \mathcal{L}_{KL}(p \| m) + \mathcal{L}_{KL}(q_\theta \| m) \right] \\
&= \frac{1}{2} \left[ \mathbb{E}_{y \sim p(\cdot|x)} \log \frac{p(y|x)}{m(y|x)}   +  \mathbb{E}_{y \sim q_\theta(\cdot|x)} \log \frac{q_\theta (y|x)}{m(y|x)} \right] \\
&= \frac{1}{2} \left[ \mathbb{E}_{y \sim p(\cdot|x)} \log \frac{p(y|x)}{m(y|x)}   + \sum_{i=1}^{|\mathcal{V}|} q_\theta (y_i|x) \frac{p (y_i|x)}{p (y_i|x)} \log \frac{q_\theta (y_i|x)}{m(y_i|x)} \right] \\
&= \frac{1}{2} \left[ \mathbb{E}_{y \sim p(\cdot|x)} \log \frac{p(y|x)}{m(y|x)}   +  \mathbb{E}_{y \sim p(\cdot|x)} \frac{q_\theta (y|x)}{p(y|x)} \log \frac{q_\theta (y|x)}{m(y|x)} \right] \\
&= \frac{1}{2} \mathbb{E}_{p} \left[ \log \frac{2p}{p+q_\theta} +  \frac{q_\theta}{p} \log \frac{2q_\theta}{p+q_\theta} \right] \\
&= \frac{1}{2} \mathbb{E}_{p} \left[ \log \frac{2}{1+c_\theta} + c_\theta \log \frac{2c_\theta}{1+c_\theta} \right] \\
&= \frac{1}{2} \mathbb{E}_{p} \left[ c_\theta \log c_\theta - (1+c_\theta) \log \frac{1+c_\theta}{2} \right]
\end{align}

\subsection{Theoretical Analysis of $\mathcal{L}_{JD}$} \label{sec:JD_derivation}

\begin{definition}[Jeffreys Divergence]
We define the symmetric objective function $\mathcal{L}_{JD}(\theta)$ as the sum of the Forward and Reverse Kullback-Leibler divergences~\cite{jeffreys1948theory}:
\begin{equation}\label{eq:jd}
    \mathcal{L}_{JD}(p,q_\theta) = \mathcal{L}_{KL}(p || q_{\theta}) + \mathcal{L}_{(R)KL}(q_{\theta} || p)
\end{equation}
\end{definition}

\begin{figure}[h]
    \centering
    \includegraphics[width=0.4\linewidth]{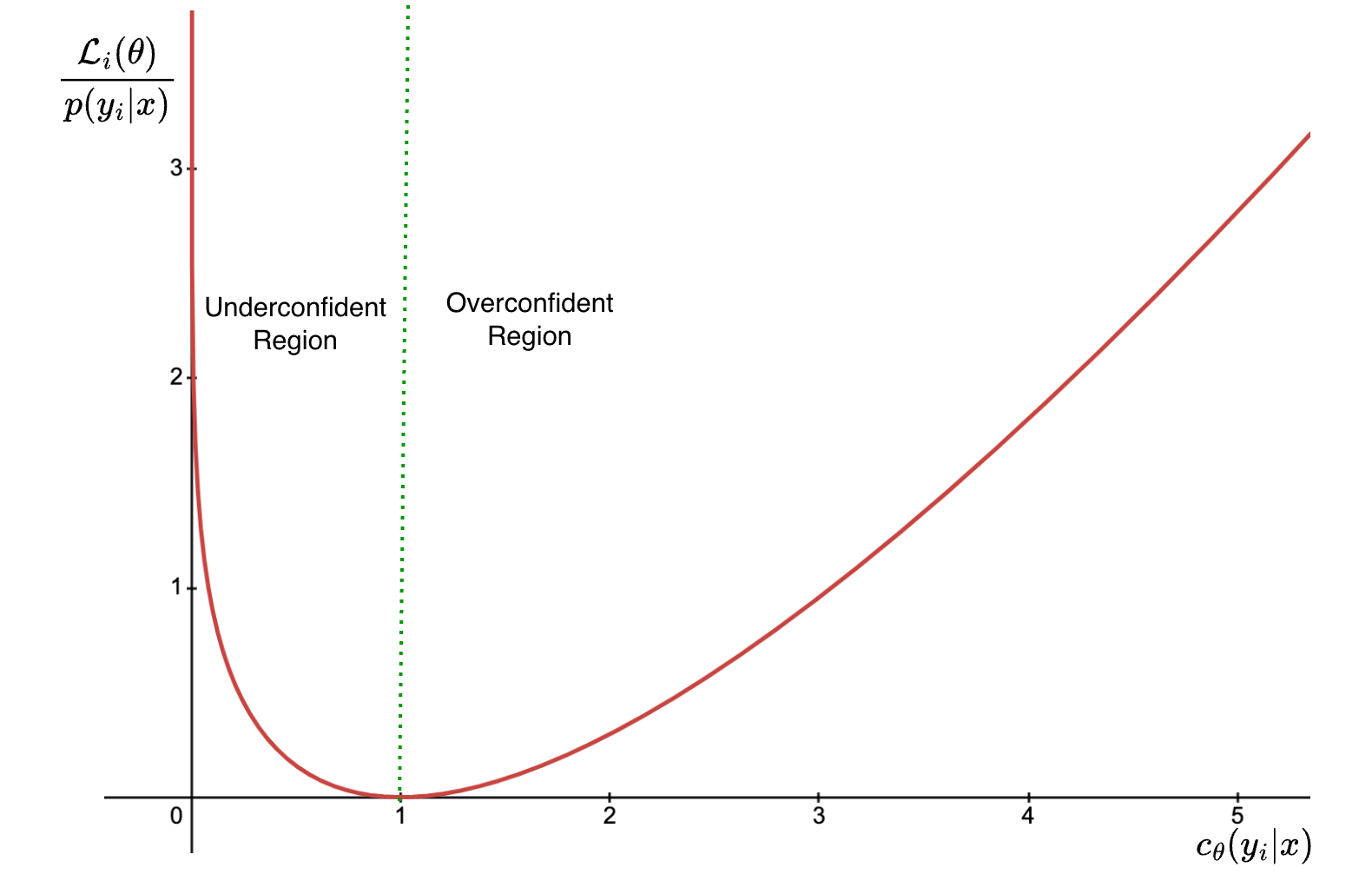}
    \caption{ The loss landscape $\mathcal{L}_{JD}$ $vs.$ the confidence score $c_\theta (y | x )$.}
    \label{fig:fkl+rkl}
\end{figure}

\begin{proposition}[Dual-sided Confidence Penalization]
The Jeffreys Divergence objective $\mathcal{L}_{JD}$ minimizes divergence through a strict, convex symmetric loss landscape. It penalizes overconfidence ($c_{\theta} \to \infty$) super-linearly while enforcing an unbounded barrier penalty for underconfidence ($c_{\theta} \to 0$), effectively trying to perfectly align with the teacher's supervision ($c_{\theta} \to 1$).
\end{proposition}

\begin{proof}
For brevity, we omit the conditioning on the input $x$ and output $y$. Substituting the definition of confidence score $c_\theta (y|x) = \frac{q_\theta(y|x)}{p(y|x)}$ into the expanded form of $\mathcal{L}_{JD}$, we derive:

\begin{align}
    \mathcal{L}_{JD}(p,q_\theta) &= \mathcal{L}_{KL}(p \| q_\theta) + \mathcal{L}_{KL}(q_\theta \| p) \\
        &= \mathbb{E}_{y \sim p(\cdot|x)} \left[ \log \frac{p(y|x)}{q_\theta(y|x)} \right] + \mathbb{E}_{y \sim q_\theta(\cdot|x)} \left[ \log \frac{q_\theta(y|x)}{p(y|x)} \right] \\
        &= \mathbb{E}_{y \sim p(\cdot|x)} \left[ \log \frac{p(y|x)}{q_\theta(y|x)} \right] + \sum_{i=1}^{|\mathcal{V}|} q_\theta (y_i|x) \frac{p (y_i|x)}{p (y_i|x)} \log \frac{q_\theta (y_i|x)}{p(y_i|x)} \\
        &= \mathbb{E}_{y \sim p(\cdot|x)} \left[ \log \frac{p(y|x)}{q_\theta(y|x)} + \frac{q_\theta (y|x)}{p(y|x)} \log \frac{q_\theta (y|x)}{p(y|x)} \right] \\
        &= \mathbb{E}_{p} \left[ \log \frac{1}{c_\theta} + c_\theta \log c_\theta \right] \\
        &= \mathbb{E}_{p} \Big[ \underbrace{(c_\theta - 1) \log c_\theta}_{g(c_\theta)} \Big]
\end{align}

To analyze the optimization landscape, we decompose the objective into the per-class loss function $g(c_\theta)$. We treat the teacher probability $p$ as a static scaling factor and analyze the behavior of $g(c_\theta)$ in three distinct regimes:

\textbf{Case 1: Overconfidence ($c_{\theta} \to \infty$).} As the student assigns excessive probability mass relative to the teacher, the term is dominated by $c_\theta \log c_\theta$. Unlike JSD (which is linear), JD applies a severe penalty to hallucinations:
\[
    \lim_{c_{\theta} \to \infty} g(c_\theta) \approx c_\theta \log c_\theta \quad \text{(Super-Linear Penalty)}
\]

\textbf{Case 2: Underconfidence ($c_{\theta} \to 0$).} As the student fails to capture the teacher's probability mass, the term is dominated by $-\log c_\theta$. Unlike JSD (which is bounded), JD enforces an infinite penalty for missed recall:
\[
    \lim_{c_{\theta} \to 0} g(c_\theta) \approx \infty \quad \text{(Unbounded Tail Sensitivity)}
\]

\textbf{Case 3: Perfect Alignment ($c_{\theta} = 1$).} When the student perfectly matches the teacher ($q_\theta = p$), the loss vanishes, confirming $c_\theta=1$ as the global minimum:
\[
    g(1) = (1 - 1) \log 1 = 0
\]

Thus, $\mathcal{L}_{JD}$ enforces a convex basin that prevents the student from drifting into extreme regions, strictly driving the confidence score towards the equilibrium at $c_{\theta}=1$.
\end{proof}

\section{Experiments}

\subsection{Training Details}
\label{app:sec:training_details}

We conduct all knowledge distillation experiments using \texttt{PyTorch} on a computational node equipped with 4 NVIDIA A100 (40GB) GPUs. 
To maximize training throughput while maintaining numerical stability, we utilize Brain Floating Point (BF16) mixed precision.
Additional details are provided in the \cref{app:tab:hyperparameters}.

\begin{table}[h]
    \centering
    \small
    \renewcommand{\arraystretch}{1.2}
    \caption{\textbf{Hyperparameter settings.} Comparison of training configurations across different student model sizes.}
    \begin{tabular}{l|c|c|c}
    \toprule
    \textbf{Hyperparameter} & \textbf{GPT-2 (120M)} & \textbf{GPT-2 (340M)} & \textbf{TinyLlama (1.1B)} \\
    \midrule
    Batch Size (per GPU) & 4 & 4 & 1 \\
    Initial Learning Rate (LR) & $2.0 \times 10^{-4}$ & $5.0 \times 10^{-5}$ & $5.0 \times 10^{-6}$ \\
    LR Decay Style & Cosine & Cosine & Cosine \\
    Final Learning Rate & $1.0 \times 10^{-7}$ & $1.0 \times 10^{-7}$& $1.0 \times 10^{-7}$ \\
    Optimizer & AdamW & AdamW & AdamW \\
    Weight Decay & 0.01 & 0.01 & 0.01 \\
    Parallel Strategy & DDP & DDP & DDP \\
    Teacher Logit Lens Layers & \{8, 16, 24, 32, 40\} & \{8, 16, 24, 32, 40\} & $\{5, 10, 16, 21, 26\}$ \\
    Student Logit Lens Layers & \{2, 4, 6, 8, 10\} & $\{4, 8, 12, 16, 20\}$ & $\{4, 7, 11, 15, 18\}$ \\
    \bottomrule
    \end{tabular}
    \label{app:tab:hyperparameters}
\end{table}

\subsection{Ablation Study}
\label{app:sec:ablation_study}
\paragraph{Scaling Factor $\mathbf{\lambda}$.}
\begin{table}[t]
    \centering
    \caption{\textbf{R-L Score.} \textsc{DistillLens} performs best in average with scaling factor $\lambda=1.0$.}
    \resizebox{0.25\linewidth}{!}{
        \begin{tabular}{l|ccc}
        \toprule
        \textbf{$\lambda$} & \textbf{Dolly} & \textbf{S-NI} & \textbf{Un-NI}\\
        \midrule
        0.1 & 24.9 & 22.7 & 26.3 \\
        0.5 & 24.7 & 25.5 & 27.1 \\ 
        1.0 & 25.2 & 24.3 & 27.9\\
        5.0 & 24.7 & 22.8 & 26.9\\
        10.0 & 25.0 & 23.7 & 26.7\\
        \bottomrule
        \end{tabular}
    }
    \label{tab:lambda_ablate}
\end{table}
We investigate the sensitivity of \textsc{DistillLens} to the scaling factor $\lambda$, which balances the intermediate alignment loss with the primary task objective. 
As shown in \Cref{tab:lambda_ablate}, performance generally improves as $\lambda$ increases from 0.1, peaking at $\lambda=1.0$ where the student achieves the highest scores on Dolly (25.2) and Un-NI (27.9). 
However, further increasing the weight to $\lambda=5.0$ leads to performance degradation, suggesting that excessive intermediate regularization can overwhelm the final task supervision. 
Consequently, we adopt $\lambda=1.0$ as the optimal setting for all main experiments, as it provides robust structural guidance without distracting from the generation task.

\subsection{Quantifying Exposure Bias}
\label{app:sec:exposure_bias}

To rigorously measure the impact of exposure bias on generation quality, we adopt the \textit{Excess Accumulated Error} (ExAccErr) metric proposed by \citet{arora2022exposure} and modified by \citet{gu2024minillm}. This metric isolates the performance degradation caused specifically by the student's own generated history (exposure bias) from its intrinsic modeling error.

We first define the \textbf{Accumulated Regret}, $R(l)$, which measures the student $q_\theta$'s divergence from the teacher $p$ when generating a sequence of length $l$ autoregressively (free-run generation):
\begin{equation}
    R(l) = \sum_{t=1}^l \mathbb{E}_{\substack{\mathbf{y}_{<t} \sim q_\theta(\cdot | \mathbf{x}) \\ y_t \sim p(\cdot | \mathbf{y}_{<t}, \mathbf{x}) }} \left[ \log \frac{p(y_t | \mathbf{y}_{<t}, \mathbf{x})}{q_\theta(y_t | \mathbf{y}_{<t}, \mathbf{x})} \right].
\end{equation}

Next, we define the \textbf{Oracle Error Rate}, $\epsilon(l)$, which serves as a baseline. It measures the average per-step divergence when the student is provided with the \textit{perfect} context (sampled from the teacher) at every step:
\begin{equation}
    \epsilon(l) = \frac{1}{l} \sum_{t=1}^l \mathbb{E}_{\substack{\mathbf{y}_{<t} \sim p(\cdot | \mathbf{x}) \\ y_t \sim p(\cdot | \mathbf{y}_{<t}, \mathbf{x}) }} \left[ \log \frac{p(y_t | \mathbf{y}_{<t}, \mathbf{x})}{q_\theta(y_t | \mathbf{y}_{<t}, \mathbf{x})} \right].
\end{equation}

The term $l\epsilon(l)$ represents the total error expected if no exposure bias existed. The difference $R(l) - l\epsilon(l)$ therefore captures the \textit{excess} error attributable solely to the drift caused by the student's self-generated prefixes. The final metric normalizes this excess error as a percentage:
\begin{align}
    \text{ExAccErr}(l) = \frac{R(l) - l\epsilon(l)}{l\epsilon(l)} \times 100\%.
\end{align}
A lower ExAccErr indicates that the model is more robust to its own generation errors, effectively maintaining alignment with the teacher even as the sequence length increases.

\subsection{SBERT Similarity Score}
\label{app:sec:sbert_similarity_score}
\input{tabs/bert_score}
\begin{figure}[t]
    \centering
    \includegraphics[width=0.8\linewidth]{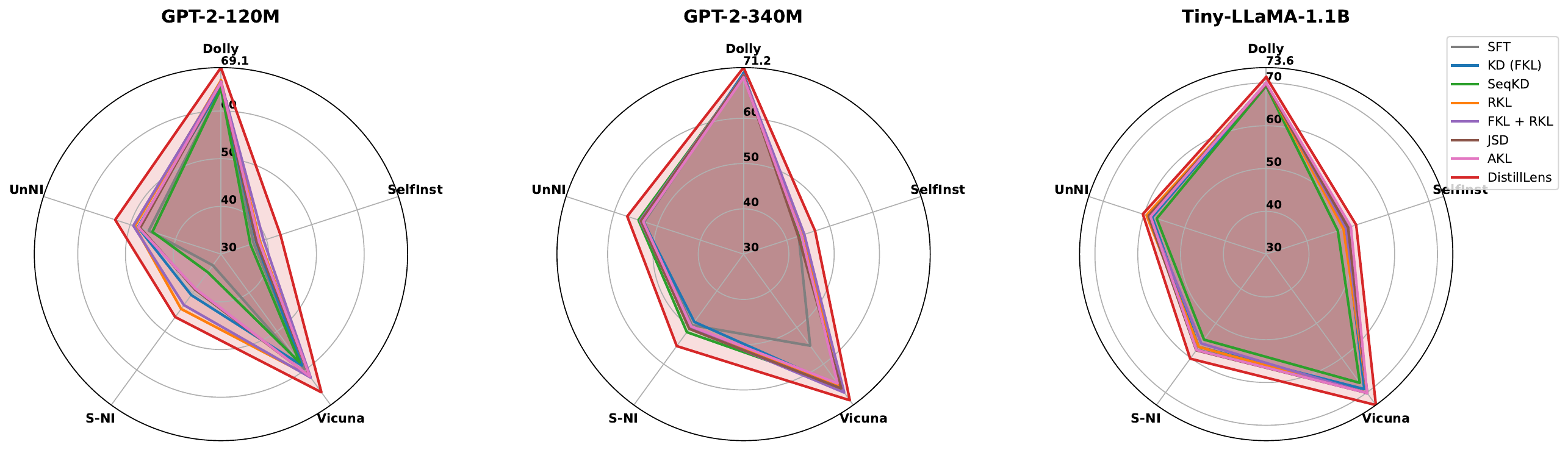}
    \caption{\textbf{Semantic Similarity Landscape.} Radar charts visualizing SBERT similarity scores across five diverse instruction-following benchmarks. \textsc{DistillLens} demonstrates robust generalization, consistently covering a larger area than baselines (SFT and Standard KD).}
    \label{fig:fkl+rkl}
\end{figure}

To evaluate the semantic quality of the generated responses, we employ a semantic textual similarity metric based on Sentence-BERT (SBERT)~\cite{reimers2019sentence}. While n-gram metrics like BLEU or ROUGE focus on surface-level lexical overlap, SBERT allows us to measure the semantic proximity of the generated instructions to the ground truth references, which is crucial for open-ended instruction following tasks.

We utilize the \texttt{all-mpnet-base-v2} model from the \texttt{sentence-transformers} library, which maps sentences to a 768-dimensional dense vector space. For a given input prompt, let $y$ be the ground truth response and $\hat{y}$ be the model's generated response. We compute the embeddings $u = \text{SBERT}(y)$ and $v = \text{SBERT}(\hat{y})$. The similarity score is calculated as the cosine similarity between these normalized embeddings:

\begin{equation}
    \text{Score}(y, \hat{y}) = \frac{u \cdot v}{\|u\| \|v\|} \times 100 \%
\end{equation}
To ensure the robustness of our results, we report the average similarity score across five distinct random seeds $\{10, 20, 30, 40, 50\}$ for each experimental configuration.

Table~\ref{tab:results} presents the BERT similarity scores across three model architectures: GPT-2 (120M and 340M) and TinyLlama-1.1B. The results demonstrate the efficacy of our proposed method, \textsc{DistillLens}, compared to Supervised Fine-Tuning (SFT) and standard Knowledge Distillation (KD) baselines.

\paragraph{Performance Superiority:} 
Across all tested models and datasets, \textsc{DistillLens} consistently achieves the highest semantic similarity scores. For instance, on the TinyLlama-1.1B model, \textsc{DistillLens} outperforms the standard SFT baseline by an average of 2.62 points (63.52 vs. 60.90) and the standard Forward KL (FKL) distillation by roughly 3 points. This trend is maintained in the smaller GPT-2 120M model, where our method achieves a score of 55.52 compared to 47.90 for SFT, indicating that our approach is particularly effective at compressing knowledge into smaller architectures.


\paragraph{Scalability:}
The results also validate that the distillation gains scale with model size. While the move from GPT-2 120M to 340M yields a performance jump of approximately 4-6 points on average, the transition to the TinyLlama-1.1B architecture further pushes the ceiling. Crucially, \textsc{DistillLens} allows the GPT-2 340M model (Avg 59.96) to approach the performance of the significantly larger TinyLlama-1.1B SFT baseline (Avg 60.90), suggesting that our distillation technique effectively bridges the gap between varying model capacities.

\subsection{GPT-4o-mini as Judge}
\label{app:sec:gpt4o_mini_judge}
\begin{figure}[t]
    \centering
    \begin{tcolorbox}[colback=gray!5!white, colframe=black!75, title=\textbf{GPT-4o-mini Evaluation Prompt}, fonttitle=\bfseries, boxrule=0.8pt]
    \small
    \textbf{\#\#\# System Prompt:}\\
    You are an intelligent and impartial evaluator. Your task is to rate the quality of a generated response by comparing it to a ground truth reference answer, based on a specific user instruction.
    
    \vspace{0.2cm}
    \textbf{\#\#\# Evaluation Criteria:}
    \begin{enumerate}[leftmargin=*, noitemsep, topsep=0pt]
        \item \textbf{Accuracy:} Does the predicted response contain the same factual information as the ground truth?
        \item \textbf{Completeness:} Does it miss any crucial details present in the ground truth?
        \item \textbf{Hallucination:} Does it include false information not supported by the ground truth or general knowledge?
        \item \textbf{Tone/Format:} Does it follow the constraints of the instruction?
    \end{enumerate}

    \vspace{0.2cm}
    \vspace{0.2cm}
    \textbf{\#\#\# Scoring Scale (1--10)}
    \begin{description}[leftmargin=*, style=nextline]
        \item[\textbf{1: Completely Incorrect}] The response is unrelated, incoherent, or factually contradictory to the ground truth.
        \item[\textbf{2--3: Severe Issues}] The response is on topic but misses the majority of key points, contains significant hallucinations, or completely fails the formatting constraints.
        \item[\textbf{4--5: Mediocre}] The response captures the general gist but misses important nuances, adds significant fluff, or contains minor factual errors.
        \item[\textbf{6--7: Acceptable/Good}] The response is accurate and captures the main idea. It may have minor omissions, slight formatting issues, or awkward phrasing, but the core information is correct.
        \item[\textbf{8--9: Very Good/Excellent}] The response is semantically similar to the ground truth, highly accurate, and follows instructions well. Only very minor stylistic differences exist.
        \item[\textbf{10: Perfect}] The response is semantically equivalent to the ground truth (or better); it captures all key information perfectly with optimal phrasing and strict adherence to formatting.
    \end{description}

    \vspace{0.2cm}
    \textbf{\#\#\# Output Format}\\
    First, provide a brief explanation of your reasoning (2-3 sentences).
    Then, end your response with the score in this format: Score: X
    
    \vspace{0.2cm}
    \hrule
    \vspace{0.2cm}
    
    \textbf{\#\#\# Input Data}\\
    \textbf{Instruction:}\\
    \texttt{\{instruction\}}
    \vspace{0.2cm}
    
    \textbf{Ground Truth:}\\
    \texttt{\{ground\_truth\}}
    \vspace{0.2cm}
    
    \textbf{Predicted Response:}\\
    \texttt{\{prediction\}}
    
    \vspace{0.2cm}
    \textbf{\#\#\# Output}
    \end{tcolorbox}
    \caption{The exact prompt template used for the GPT-4o-mini judge evaluation.}
    \label{fig:prompt_gpt4o}
\end{figure}

To evaluate the semantic quality and instruction-following capabilities of the finetuned models, we employ a model-based evaluation approach using \texttt{gpt-4o-mini}~\cite{hurst2024gpt} as a judge. While traditional metrics like ROUGE provide a measure of lexical overlap, they often fail to capture semantic nuances, logical consistency, and adherence to complex instructions. A strong LLM judge correlates better with human judgment for open-ended generation tasks~\cite{zheng2023judging}.

We design a rigorous evaluation prompt, shown in Figure~\ref{fig:prompt_gpt4o}, which instructs the judge to assess the predicted response against a ground truth reference based on four key criteria: Accuracy, Completeness, Hallucination, and Tone/Format. The judge assigns a score on a Likert scale from 1 to 10, accompanied by a brief reasoning statement to ensure interpretability.

To align the evaluation metrics with a standard percentage scale, we normalize the raw GPT-4 scores ($S_{\text{raw}} \in [1, 10]$) by zero-indexing the value and scaling it by the maximum possible range:
\begin{equation}
    S_{\text{norm}} = \frac{S_{\text{raw}} - 1}{9} \times 100\%
\end{equation}
This transformation results in a final score range of $S_{\text{norm}} \in [0, 100]$, ensuring that the lowest qualitative rating (1) corresponds to $0\%$ and a perfect rating (10) corresponds to $100\%$. This normalization facilitates direct comparison with other percentage-based metrics and simplifies the interpretation of relative performance gains.

For every test sample, we query \texttt{gpt-4o-mini} with the instruction, the ground truth reference, and the model's generated output. We set the judge's sampling temperature to 0.7. The final reported score for each model is the average $S_{\text{norm}}$ across the entire test set.


\end{document}

%% file: tabs/main_results.tex
\begin{table*}[!t]
    \centering
    \small
    \caption{\textbf{Evaluation results.} Average Rouge-L (R-L) and GPT-4o (mini) feedback scores across 5 random seeds across $\{10,20,30,40,50\}$. The best scores of each model size are \textbf{boldfaced} and the second best are \underline{underlined}. }
    \begin{tabular}{lrl|cc|cc|cc|c|c||c}
    \toprule
    \multirow{2}{*}{\textbf{Model}} & \multirow{2}{*}{\textbf{\#Params}} & \multirow{2}{*}{\textbf{Method}}
    & \multicolumn{2}{c|}{\textbf{Dolly}} & \multicolumn{2}{c|}{\textbf{SelfInst}}  & \multicolumn{2}{c|}{\textbf{Vicuna}} & \textbf{S-NI} & \textbf{UnNI} & \textbf{Avg}\\
    \cmidrule{4-12}
      & &  &  R-L & GPT-4o & R-L & GPT-4o & R-L & GPT-4o & R-L & R-L & R-L    \\ 
     \midrule
     \multirow{20}{*}{GPT-2}
      & 1.5B & Teacher &  27.6 & 23.9 & 14.3 & 12.5 & 16.3 & 16.0 & 27.6 & 31.8 & 23.52    \\ 
      \cmidrule{2-12}

     & \multirow{8}{*}{120M}
        & SFT &  23.3 & 13.4 & 10.0 & 4.9 & 14.7 & 6.7 & 16.3 & 18.5 & 16.56 \\
     &  & KD (FKL) &  23.5 & 14.3 & 10.3 & 5.7 & 14.7 & 6.7 & 16.6 & 20.9  & 17.20  \\
     &  & SeqKD &  22.7 & 13.9 & 10.1 & 4.8 & 14.3 & 6.8 & 16.4 & 18.8 & 16.64 \\
     & & RKL & 23.2 & 14.4 & \underline{10.6} & 5.9 & 14.7 & 6.8 & \underline{17.9} & \underline{21.4} & 17.56 \\
     & & FKL + RKL & 23.5 & 14.5 & 10.4 & 6.1 & 15.0 & \textbf{7.4} & 17.6 & 21.0 & 17.50 \\
     & & JSD & 23.4 & 14.6 & 10.4 & 6.2 & 14.5 & 7.0 & 17.1 & 20.6 & 17.20 \\
     & & AKL & \underline{23.7} & 14.3 & 10.4 & \underline{6.4} & \underline{15.3} & 6.9 & 17.6 &  21.2 & \underline{17.64} \\
     & & \textsc{DistillLens} & \textbf{25.2} & \textbf{15.0} & \textbf{12.4} & \textbf{7.0} & \textbf{15.8} & \underline{7.3} & \textbf{24.3} & \textbf{27.9} & \textbf{21.12} \\
     \cmidrule{2-12}
     & \multirow{8}{*}{340M}
        & SFT & 25.3 & \underline{19.8} & 12.3 & 9.2 & 16.0 & 12.5 & 22.7 & 26.6 & 20.58 \\
     &  & KD (FKL)  & \underline{25.5} & 19.6 & 12.1 & 8.5 & 16.1 & 12.3 & 21.9 & 26.5 & 20.42 \\
     &  & SeqKD & 25.2 & 19.6 & 12.1 & 8.9 & \underline{16.2} & 12.6 & \underline{23.9} & \underline{27.8} & \underline{21.04} \\
     & & RKL & 24.9 & 19.5 & \underline{13.3} & 9.4 & 16.0 & 12.4 & 23.4 & 27.5 & 21.02\\
     & & FKL + RKL & 24.9 & 19.6 & 12.9 & \underline{10.0} & 16.1 & \underline{12.8} & 23.2 & 27.0 & 20.82 \\
     & & JSD & 25.1 & 19.6 & 12.0 & 9.3 & 15.8 & 11.7 & 23.2 & 27.0 & 20.62 \\
     & & AKL & 24.8 & 19.3 & 12.7 & 9.1 & 15.2 & 12.7 & 23.2 &  27.3 & 20.64 \\
     & & \textsc{DistillLens} & \textbf{26.4} & \textbf{20.1} & \textbf{14.6} & \textbf{10.3} & \textbf{16.5} & \textbf{13.0} & \textbf{28.1} & \textbf{33.0} & \textbf{23.72} \\
    \midrule
    \multirow{9}{*}{Llama}
      & 6.7B & Teacher &  26.3 & 33.1 & 20.8 & 25.1 & 17.5 & 26.1 & 32.4 & 35.8 & 26.56    \\ 
      \cmidrule{2-12}
     & \multirow{8}{*}{1.1B}
        & SFT & 25.5 & 23.1 & 17.1 & 15.7 & 16.9 & 16.1 & 29.5 & 31.8 & 24.16  \\
     &  & KD (FKL) &  25.3 & 22.9 & 17.0 & 16.8 & 16.9 & 16.4 & 28.8 & 31.1  & 23.82  \\
     &  & SeqKD &  24.9 & 22.3 & 16.2 & 14.8 & 16.5 & 15.6 & 27.7 & 30.6 & 23.18   \\
     & & RKL & \underline{25.6} & 23.3 & 16.3 & 16.0 & \underline{17.7} & \underline{17.8} & 28.4 & 32.0 & 24.00 \\
     & & FKL + RKL & 25.5 & 23.0 & \underline{17.5} & \underline{17.2} & 17.1 & 17.3 & 29.9 & \underline{32.8} & \underline{24.56} \\
     & & JSD & 25.4 & \underline{23.7} & 16.9 & 16.9 & 17.4 & 16.3 & \underline{30.0} & 32.5 & 24.44   \\
     & & AKL & 25.3 & \underline{23.7} & 17.4 & 16.9 & 17.3 & 17.3 & 29.5 & 32.1 & 24.32 \\
     & & \textsc{DistillLens} & \textbf{25.9} & \textbf{24.5} & \textbf{18.2} & \textbf{19.0} & \textbf{18.2} & \textbf{20.8} & \textbf{30.8} & \textbf{34.3} & \textbf{25.48} \\
     \bottomrule
    \end{tabular}
    \label{tab:main_results}
\end{table*}

%% file: tabs/bert_score.tex
\begin{table*}[!t]
    \centering
    \small
    \caption{\textbf{SBERT Similarity Scores.} Quantitative comparison of semantic alignment across three student models. \textsc{DistillLens} consistently achieves higher semantic similarity to ground truth references than SFT and Standard KD.}
    \begin{tabular}{ll|c|c|c|c|c||c}
    \toprule
    \textbf{Model} & \textbf{Method}
    & \textbf{Dolly} & \textbf{SelfInst}  & \textbf{Vicuna} & \textbf{S-NI} & \textbf{UnNI} & \textbf{Avg}\\
     \midrule
    \multirow{8}{*}{GPT-2 120M}
        & SFT & 64.6 & 37.2 & 59.8 & 32.9 & 45.9 & 47.90 \\
     & KD (FKL) &  65.2 & 37.8 & 58.9 & 40.6 & 48.3  & 50.16  \\
     & SeqKD &  64.7 & 36.5 & 58.1 & 34.7 & 45.1 & 47.82 \\
     & RKL & 66.4 & 38.6 & 61.6 & 44.2 & 49.0 & 51.96 \\
     & FKL + RKL & 66.1 & 39.4 & 62.0 & 43.2 & 49.3 & 52.00 \\
     & JSD & 65.9 & 37.9 & 61.6 & 39.2 & 47.8 & 50.48 \\
     & AKL & 66.0 & 38.4 & 61.7 & 39.0 &  48.2 & 50.66 \\
     & \textsc{DistillLens} & \textbf{69.1} & \textbf{43.1} & \textbf{65.8} & \textbf{46.3} & \textbf{53.3} & \textbf{55.52} \\
     \midrule
    \multirow{8}{*}{GPT-2 340M}
        & SFT & 69.6 & 42.8 & 55.0 & 49.4 & 53.1 & 53.98 \\
     & KD (FKL)  & 70.0 & 43.1 & 66.2 & 48.5 & 53.3 & 55.28 \\
     & SeqKD & 69.4 & 43.2 & 65.9 & 51.3 & 54.4 & 56.84 \\
     & RKL & 69.4 & 44.0 & 65.9 & 49.8 & 54.0 & 56.62\\
     & FKL + RKL & 69.5 & 44.1 & 67.8 & 50.1 & 54.0 & 57.00 \\
     & JSD & 69.4 & 42.7 & 66.7 & 50.4 & 53.8 & 56.60 \\
     & AKL & 69.0 & 43.4 & 65.3 & 49.7 &  53.3 & 56.14 \\
     & \textsc{DistillLens} & \textbf{71.2} & \textbf{46.6} & \textbf{69.9} & \textbf{55.1} & \textbf{57.0} & \textbf{59.96} \\
    \midrule
    \multirow{8}{*}{Tiny-Llama-1.1B}
        & SFT & 69.4 & 50.2 & 69.0 & 57.6 & 58.3 & 60.90  \\
     & KD (FKL) &  69.6 & 49.4 & 69.0 & 56.8 & 58.0  & 60.56  \\
     & SeqKD &  69.3 & 47.6 & 67.2 & 54.7 & 56.9 & 59.14    \\
     & RKL & 69.8 & 49.0 & 70.1 & 56.8 & 59.4 & 61.02 \\
     & FKL + RKL & 70.0 & 49.6 & 70.2 & 55.8 & 58.7 & 60.86 \\
     & JSD & 69.6 & 50.1 & 70.1 & 57.8 & 58.9 & 61.30 \\
     & AKL & 70.0 & 50.8 & 70.1 & 57.7 & 58.4 & 62.00 \\
     & \textsc{DistillLens} & \textbf{71.4} & \textbf{52.1} & \textbf{73.6} & \textbf{60.2} & \textbf{60.3} & \textbf{63.52} \\
     \bottomrule
    \end{tabular}
    \vspace{0.4cm}
    \vspace{-0.5cm}
    \label{tab:results}
\end{table*}